\documentclass{article}

\usepackage{microtype}
\usepackage{graphicx}
\usepackage{booktabs} 
\usepackage{amsthm}
\usepackage{amsmath}
\usepackage{amsfonts}

\usepackage{array}
\usepackage{amsmath}
\usepackage{graphicx}
\usepackage{makecell}
\usepackage{multicol, blindtext}
\usepackage{soul,color,xcolor}
\usepackage{multirow}
\usepackage[english]{babel}
\usepackage{amssymb}
\usepackage{tabularx}
\usepackage{bm}
\usepackage{comment}
\usepackage{float}
\usepackage{algorithm}
\usepackage{caption}
\usepackage{subfig}
\usepackage{wrapfig}

\newtheorem{theorem}{Theorem}
\newtheorem{lemma}{Lemma}
\newtheorem{assumption}{Assumption}
\newtheorem{definition}{Definition}

\newcommand{\vc}[1]{\mathbf{#1}}
\newcommand{\ci}{\perp\!\!\!\perp}
\newcommand{\cl}[1]{\mathcal{#1}}
\newcommand{\bb}[1]{\mathbb{#1}}
\newcommand{\hist}[1]{\bar{\vc{#1}}}
\newcommand{\histcl}[1]{\bar{\cl{#1}}}

\usepackage{hyperref}

\usepackage[accepted]{icml2020}

\icmltitlerunning{Time Series Deconfounder: Estimating Treatment Effects over Time in the Presence of Hidden Confounders}

\begin{document}

\twocolumn[
\icmltitle{Time Series Deconfounder: Estimating Treatment Effects over Time in the Presence of Hidden Confounders}



\begin{icmlauthorlist}
	\icmlauthor{Ioana Bica}{ox,turing}
	\icmlauthor{Ahmed M. Alaa}{ucla}
	\icmlauthor{Mihaela van der Schaar}{turing,ucla,cam}
\end{icmlauthorlist}

\icmlaffiliation{ox}{University of Oxford, Oxford, United Kingdom}
\icmlaffiliation{turing}{The Alan Turing Institute, London, United Kingdom}
\icmlaffiliation{ucla}{UCLA, Los Angeles, USA}
\icmlaffiliation{cam}{University of Cambridge, Cambridge, United Kingdom}

\icmlcorrespondingauthor{Ioana Bica}{ioana.bica@eng.ox.ac.uk}

\icmlkeywords{Machine Learning, ICML}

\vskip 0.3in
]



\printAffiliationsAndNotice{} 

\begin{abstract}
The estimation of treatment effects is a pervasive problem in medicine. Existing methods for estimating treatment effects from longitudinal observational data assume that there are no hidden confounders, an assumption that is not testable in practice and, if it does not hold, leads to biased estimates. In this paper, we develop the Time Series Deconfounder, a method that leverages the assignment of multiple treatments over time to enable the estimation of treatment effects in the presence of multi-cause hidden confounders. The Time Series Deconfounder uses a novel recurrent neural network architecture with multitask output to build a factor model over time and infer latent variables that render the assigned treatments conditionally independent; then, it performs causal inference using these latent variables that act as substitutes for the multi-cause unobserved confounders. We provide a theoretical analysis for obtaining unbiased causal effects of time-varying exposures using the Time Series Deconfounder. Using both simulated and real data we show the effectiveness of our method in deconfounding the estimation of treatment responses over time. 
\end{abstract}

	\section{Introduction}
	
	Forecasting the patient's response to treatments assigned over time represents a crucial problem in the medical domain. The increasing availability of observational data makes it possible to learn individualized treatment responses from longitudinal disease trajectories containing information about patient covariates and treatment assignments \citep{robins2000marginal, robins2008estimation, schulam2017reliable, lim2018forecasting, bica2020estimating}. Existing methods for estimating individualized treatment effects over time assume that all confounders---variables affecting the treatment assignments and the potential outcomes---are observed, an assumption which is not testable in practice\footnote{Since counterfactuals are never observed, it is not possible to test for the existence of hidden confounders that could affect them.} and probably not true in many situations.
	
	To understand why the presence of hidden confounders introduces bias, consider the problem of estimating treatment effects for patients with cancer. They are often prescribed multiple treatments at the same time, including chemotherapy, radiotherapy and/or immunotherapy based on their tumor characteristics. These treatments are adjusted if the tumor size changes. The treatment strategy is also changed as the patient starts to develop drug resistance \citep{vlachostergios2018treatment} or the toxicity levels of the drugs increase \citep{kroschinsky2017new}. Drug resistance and toxicity levels are multi-cause confounders since they affect not only the multiple causes (treatments)\footnote{Causes and treatments are used interchangeably throughout the paper.}, but also the patient outcome (e.g. mortality, risk factors). However, drug resistance and toxicity may not be observed and, even if observed, may not be recorded in the electronic health records. Estimating, for instance, the effect of chemotherapy on the cancer progression in the patient without accounting for the dependence on drug resistance and toxicity levels (hidden confounders) will produce biased results.  
	
	\citet{wang2018blessings} developed theory for deconfounding---adjusting for the bias introduced by the existence of hidden confounders in observational data---in the \textit{static} causal inference setting and noted that the existence of multiple causes makes this task easier. \citet{wang2018blessings} observed that the dependencies in the assignment of multiple causes can be used to infer latent variables that render the causes independent and act as substitutes for the hidden confounders. 	
	
	In this paper, we propose the Time Series Deconfounder, a method that enables the unbiased estimation of treatment responses \textit{over time} in the presence of hidden confounders, by taking advantage of the  dependencies in the sequential assignment of multiple treatments.  We draw from the main idea in \citet{wang2018blessings}, but note that the estimation of hidden confounders in the longitudinal setting is significantly more complex than in the static setting, not just because the hidden confounders may vary over time but in particular because the hidden confounders may be affected by previous treatments and covariates. Thus, standard latent variable models are no longer applicable, as they cannot capture these time dependencies. 	
	
	The Time Series Deconfounder relies on building a factor model \textit{over time} to obtain latent variables which, together with the observed variables render the assigned causes conditionally independent. Through theoretical analysis we show that these latent variables can act as substitutes for the multi-cause unobserved confounders and can be used to satisfy the sequential strong ignorability condition in the potential outcomes framework for time-varying exposures \citep{robins2008estimation} and obtain unbiased estimates of individualized treatment responses, using weaker assumptions than standard methods. Following our theory, we propose a novel deep learning architecture, based on a recurrent neural network with multi-task outputs and variational dropout, to build such a factor model and infer the substitutes for the hidden confounders in practice. 
	
	The Time Series Deconfounder shifts the need for observing all multi-cause confounders (untestable condition) to constructing a good factor model over time (testable condition). To assess how well the factor model captures the distribution of assigned treatments, we extend the use of predictive checks \citep{rubin1984bayesianly, wang2018blessings} to the temporal setting and compute $p$-values at each timestep. We perform experiments on a simulated dataset where we control the amount of hidden confounding applied and on a real dataset with patients in the ICU \citep{johnson2016mimic} to  show how the Time Series Deconfounder allows us to deconfound the estimation of treatment responses in longitudinal data. To the best of our knowledge, this represents the first method for learning latent variables that can act as substitutes for the unobserved confounders in the time series setting.
	
	\section{Related Work}
	
	Previous methods for causal inference mostly focused on the static setting \citep{hill2011bayesian, wager2017estimation, alaa2017bayesian, shalit2017estimating,  yoon2018ganite, alaa2018limits,zhang2020learning, bica2020estimatingdose}, and less attention has been given to the time series setting.  We discuss methods for estimating treatment effects over time, as well as methods for inferring substitute hidden confounders in the static setting. 
	
	\textbf{Potential outcomes for time-varying treatment assignments}.  
	Standard methods for performing counterfactual inference in longitudinal data are found in the epidemiology literature and include the g-computation formula, g-estimation of structural nested mean models, and inverse probability of treatment weighting of marginal structural models \citep{robins1994correcting, robins2000marginal, robins2008estimation}. Alternatively, \cite{lim2018forecasting} improves on the standard marginal structural models by using recurrent neural networks to estimate the propensity weights and treatment responses, while \cite{bica2020estimating} propose using balancing representations to handle the time-dependent confounding bias when estimating treatment effects over time. 
    Despite the wide applicability of these methods in forecasting treatment responses, they are all based on the assumption that there are no hidden confounders.  Our paper proposes a method for deconfounding such outcome models, by inferring substitutes for the hidden confounders which can lead to unbiased estimates of the potential outcomes.
	
    The potential outcomes framework has been extended to the continuous-time setting by \cite{lok2008statistical}. Several methods have been proposed for estimating treatment responses in continuous time \citep{soleimani2017treatment, schulam2017reliable}, again assuming that there are no hidden confounders. Here, we focus on deconfounding the estimation of treatment responses in the discrete-time setting.
	
    Sensitivity analysis methods that evaluate the potential impact that an unmeasured confounder could have on the estimation of treatment effects have also been developed \cite{robins2000sensitivity, roy2016bayesian, scharfstein2018global}. However, these methods assess the suitability of applying existing tools, rather than propose a direct solution for handling the presence of hidden confounders in observational data. 
	
    \textbf{Latent variable models for estimating hidden confounders.} The most similar work to ours is the one of \citet{wang2018blessings}, who proposed the deconfounder, an algorithm that infers latent variables that act as substitutes for the hidden confounders and then performs causal inference in the static multi-cause setting. The deconfounder involves finding a good factor model of the assigned causes which can be used to estimate substitutes for the hidden confounders. Then, the deconfounder fits an outcome model for estimating the causal effects using the inferred latent variables. Our paper extends the theory for the deconfounder to the time-varying treatments setting and shows how the inferred latent variables can lead to sequential strong ignorability.  To estimate the substitute confounders, \citet{wang2018blessings} used standard factor models  \citep{tipping1999probabilistic, ranganath2015deep}, which are only applicable in the static setting. To build a factor model over time, we propose an RNN architecture with multitask output and variational dropout.
	
	Several other methods have been proposed for taking advantage of the multiplicity of assigned treatments in the static setting and capture shared latent confounding \cite{tran2017implicit, heckerman2018accounting, ranganath2018multiple}. These works are based on Pearl's causal framework \cite{pearl2009causality} and use structural equation models. Alternative methods for dealing with hidden confounders in the static setting use proxy variables as noisy substitutes for the confounders \cite{lash2014good, louizos2017causal, lee2018estimation}. 
	
     A different line of research involves performing causal discovery in the presence of hidden confounders \cite{spirtes2000causation}. In this context, several methods have been proposed to perform causal graphical model structure learning with latent variables \cite{leray2008causal, jabbari2017discovery, raghu2018comparison}. However, in this paper, we are not aiming to discover causal relationships between patient covariates over time. Instead, we improve existing methods for estimating the individualized effects of time-dependent treatments by accounting for multi-cause unobserved confounders.

	\section{Problem Formulation}\label{sec:problem_formulation}
	
	Let the random variables $\vc{X}_t^{(i)} \in \cl{X}_t$ be the time-dependent covariates, $\vc{A}^{(i)}_t = [A_{t1}^{(i)} \dots A_{tk}^{(i)}]\in \cl{A}_t$ be the possible assignment of $k$ treatments (causes) at timestep $t$ and let $\vc{Y}_{t+1}^{(i)} \in \mathcal{Y}_t$ be the observed outcomes for patient $(i)$. Treatments can be either binary and/or continuous. Static features, such as genetic information, do not change our theory, and, for simplicity, we assume they are part of the observed covariates.  
	
	The observational data for patient $(i)$, also known as the patient trajectory, consists of realizations of the previously described random variables $\zeta^{(i)} = \{\vc{x}^{(i)}_{t},  \vc{a}^{(i)}_{t}, \vc{y}^{(i)}_{t+1} \}_{t=1}^{T^{(i)}}$, with samples collected for $T^{(i)}$ discrete and regular timesteps. Electronic health records consist of data for $N$ independent patients $\cl{D} = \{\tau^{(i)}\}_{i=1}^{N}$. For simplicity, we omit the patient superscript $(i)$ unless it is explicitly needed.
	
	We leverage the potential outcomes framework proposed by \citet{rubin1978bayesian} and \citet{neyman1923applications}, and extended by \citet{robins2008estimation} to take into account time-varying treatments. Let $\vc{Y}(\bar{\vc{a}})$ be the potential outcome, either factual or counterfactual, for each possible course of treatment $\bar{\vc{a}}$. 
	
	Let $\bar{\vc{A}}_{t} = (\vc{A}_1, \dots, \vc{A}_t)\in\histcl{A}_{t}$ be the history of treatments and let $\bar{\vc{X}}_{t} = (\vc{X}_1, \dots, \vc{X}_t)  \in \bar{\cl{X}}_t$ be the history of covariates until timestep $t$. For each patient, we want to estimate individualized treatment effects, i.e. potential outcomes conditional on the patient history of covariates and treatments:
	\begin{equation}
	    \mathbb{E}[\vc{Y}(\hist{a}_{\geq t}) \mid \hist{A}_{t-1}, \hist{X}_t],
	\end{equation}
	for any possible treatment plan $\hist{a}_{\geq t}$ that starts at timestep $t$ and consists of a sequence of treatments that ends just before the patient outcome $\vc{Y}$ is observed. The observational data can be used to fit a regression model to estimate $\mathbb{E} [\vc{Y} \mid \hist{a}_{\geq t}, \hist{A}_{t-1}, \hist{X}_t]$. Under certain assumptions, these estimates are unbiased so that $\mathbb{E} [\vc{Y}(\hist{a}_{\geq t}) \mid  \hist{X}_t, \hist{A}_{t-1} ] = \mathbb{E} [\vc{Y} \mid \hist{a}_{\geq t}, \hist{A}_{t-1}, \hist{X}_t]$. These conditions include Assumptions 1 and 2, which are standard among the existing methods for estimating treatment effects over time and can be tested in practice \citep{robins2008estimation}. 
	
	
	
	
	\begin{assumption}
	\textbf{Consistency}.  If $\bar{\vc{A}}_{\geq t} = \hist{a}_{\geq t}$, then the potential outcomes for following the treatment plan $\hist{a}_{\geq t}$ is the same as the observed (factual) outcome $\vc{Y}(\bar{\vc{a}}_{\geq t}) = \vc{Y}$.
	\end{assumption}
	
	\begin{assumption}
	\textbf{Positivity (Overlap)} \cite{imai2004causal}: 
	If $P( \hist{A}_{t-1} = \hist{a}_{t-1}, \hist{X}_t = \hist{x}_t)\neq 0$ then $P(\vc{A}_t = \vc{a}_t \mid \hist{A}_{t-1} = \hist{a}_{t-1}, \hist{X}_t = \hist{x}_t) > 0$ for all $\vc{a}_t$.
	\end{assumption}
	
	The positivity assumption means that at each timestep $t$, each treatment has a non-zero probability of being given to the patient. This assumption is testable in practice. 
	
	In addition to these two assumptions, existing methods also assume \textit{sequential strong ignorability}:
	\begin{equation}
	\vc{Y}(\hist{a}_{\geq t}) \ci \vc{A}_t \mid \hist{A}_{t-1}, \hist{X}_t,
	\end{equation}
	for all possible treatment plans $\hist{a}_{\geq t}$ and  for all $t\in \{0, \dots, T\}$. This condition holds if there are no hidden confounders and it cannot be tested in practice. To understand why this is the case, note that the sequential strong ignorability assumption requires the conditional independence of the treatments with all of the potential outcomes, both factual and counterfactual, conditional on the patient history. Since the counterfactuals are never observed, it is not possible to test for this conditional independence. 
	
	In this paper, we assume that there are hidden confounders.  Consequently, using standard methods for computing $\mathbb{E} [\vc{Y} \mid \hist{a}_{\geq t}, \hist{A}_{t-1}, \hist{X}_t ]$ from the dataset will result in biased estimates since the hidden confounders introduce a dependence between the treatments at each timestep and the potential outcomes ($\vc{Y}(\hist{a}_{\geq t}) \not \ci \vc{A}_t \mid \hist{A}_{t-1}, \hist{X}_t$) and therefore: 
	\begin{equation}
	    \mathbb{E} [\vc{Y}(\hist{a}_{\geq t}) \mid   \hist{A}_{t-1}, \hist{X}_t ] \neq  \mathbb{E} [\vc{Y} \mid \hist{a}_{\geq t}, \hist{A}_{t-1}, \hist{X}_t].
	\end{equation}
	
	By extending the method proposed by \citet{wang2018blessings}, we take advantage of the multiple treatment assignments at each timestep to infer a sequence of latent variables $\hist{Z}_{t} = (\vc{Z}_1, \dots, \vc{Z}_t) \in \histcl{Z}_t$ that can be used as substitutes for the unobserved confounders. We will then show how $\hist{Z}_t$ can be used to estimate the treatment effects over time. 
	
	\begin{figure*}[ht]
		\begin{center}
		\centerline{\includegraphics[width=2.0\columnwidth]{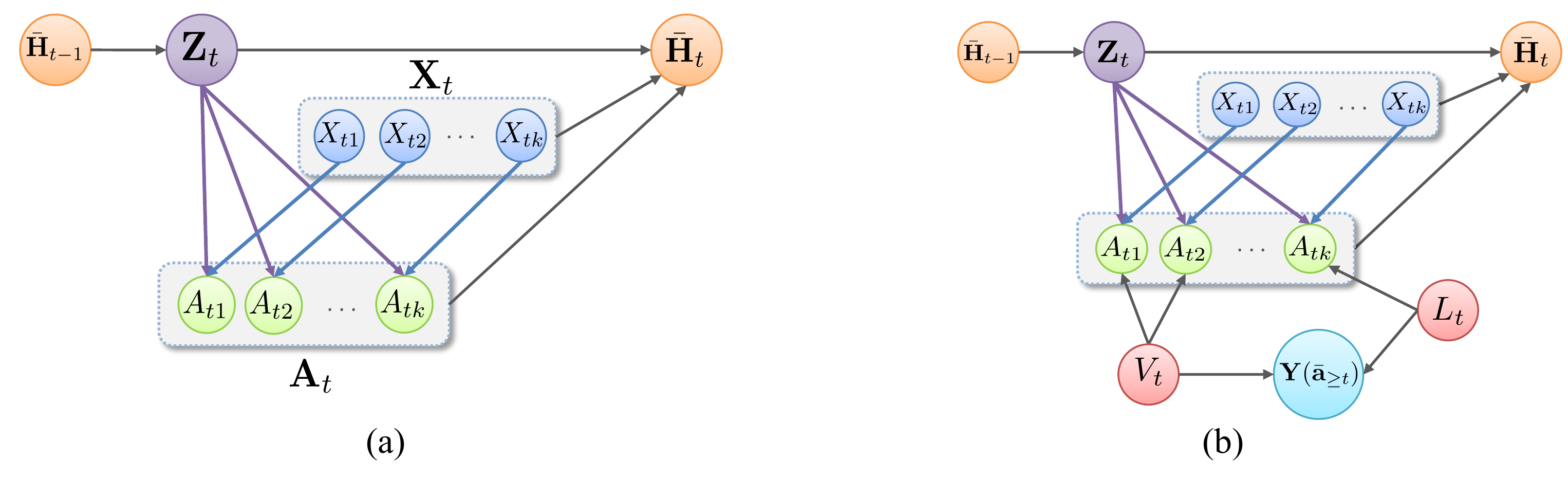}}
			\caption{(a) Graphical factor model. Each $\vc{Z}_t$ is built as a function of the history, such that, with $\vc{X}_t$, it renders the assigned causes conditionally independent: $p(a_{t1}, \dots, a_{tk}\mid \vc{z}_t, \vc{x}_t) = \prod_{j=1}^{k} p(a_{tj} \mid \vc{z}_t, \vc{x}_t)$. The variables can be connected to $\vc{Y}(\hist{a}_{\geq t})$ in any way.  (b) Graphical model explanation for why this factor model construction ensures that $\vc{Z}_t$ captures all of the multi-cause hidden confounders.  }
			\label{fig:graphical_model}
		\end{center}
		\vspace{-0.8cm}
	\end{figure*}
	
	\section{Time Series Deconfounder} \label{sec:deconfounding}
	
	The idea behind the Time Series Deconfounder is that multi-cause confounders introduce dependencies between the  treatments. As treatment assignments change over time we infer substitutes for the hidden confounders that take advantage of the patient history to capture these dependencies. 
	
	\subsection{Factor Model} 
	
	The Time Series Deconfounder builds a factor model to capture the distribution of the causes (treatments) over time.  At time $t$, the factor model constructs the latent variable $\vc{z}_t = g(\hist{h}_{t-1})$, where $\hist{h}_{t-1} = (\hist{a}_{t-1}, \hist{x}_{t-1}, \hist{z}_{t-1})$ is the realization of history $\hist{H}_{t-1}$. The latent variable $\vc{z}_t$, together with the observed patient covariates $\vc{x}_t$, render the assigned treatments conditionally independent:
	\begin{equation}
		p(a_{t1}, \dots, a_{tk}\mid \vc{z}_t, \vc{x}_t) = \prod_{j=1}^{k} p(a_{tj} \mid \vc{z}_t, \vc{x}_t).
	\end{equation}
    Figure \ref{fig:graphical_model}(a) illustrates the corresponding graphical model for timestep $t$. The factor model of the assigned treatments is built as a latent variable model with joint distribution:
	\begin{equation}
	\begin{aligned}
	p(\theta_{1:k}, \hist{x}_{T}, & \hist{z}_{T},  \hist{a}_{T})  =   \, p(\theta_{1:k}) p(\hist{x}_T) \cdot \\
	& \prod_{t=1}^{T} \big(p(\vc{z}_{t}\mid  \, \hist{h}_{t-1}) \prod_{j=1}^{k} p(a_{tj} \mid \vc{z}_t, \vc{x}_t, \theta_j) \big), 
	\end{aligned}
	\end{equation}
	where $\theta_{1:k}$ are parameters. The distribution of the treatments $p(\hist{a}_{T})$ is the corresponding marginal. Notice that we do not assume that, in the observational data, the patient covariates $\vc{x}_t$ at timestep $t$ are independent of the patient history. The graphical factor model shows how the latent variables $\vc{z}_t$ that can act as substitutes for the multi-cause unobserved confounders are built. As we will see in Section~\ref{sec:outcome_model} these latent variables will be used as part of an outcome model that estimates the potential outcomes $\vc{Y}(\hist{a}_{\geq t})$.

	By taking advantage of the dependencies between the multiple treatment assignments, the factor model allows us to infer the sequence of latent variables $\hist{Z}_t$ that can be used to render the assigned causes conditionally independent. Through this factor model construction and under correct model specifications, we can rule out the existence of other multi-cause confounders that are not captured by $\vc{Z}_t$. To understand why this is the case, consider the graphical model in  Figure \ref{fig:graphical_model}(b). By contradiction, assume that there exists another multi-cause confounder $V_t$ not captured by $\vc{Z}_t$. Then, by $d$-separation the conditional independence between the assigned causes given $\vc{Z}_t$ and $\vc{X}_t$ does not hold anymore. This argument cannot be used for single-cause confounders, such as $L_t$, which are only affecting one of the causes and the potential outcomes. Thus, we assume sequential single strong ignorability (no hidden single cause confounders). 
	
	\begin{assumption}
	\textbf{Sequential single strong ignorability:}\label{eq:single_strong_ignorability}
	\begin{equation}
	    \vc{Y}(\hist{a}_{\geq t})  \ci A_{tj} \mid \vc{X}_{t}, \bar{\vc{H}}_{t-1},
	\end{equation}
	for all $\hist{a}_{\geq t}$, for all $t\in \{0, \dots, T\}$, and for all $j \in \{1, \dots, k\}$.
	\end{assumption}
	
	Causal inference relies on assumptions. Existing methods for estimating treatment effects over time assume that there are no multi-cause and no single-cause hidden confounders. In this paper, we make the \textit{weaker} assumption that there are no single-cause hidden confounders. While this assumption is also untestable in practice, as the number of treatments increases for each timestep, it becomes increasingly weaker: the more treatments we observe, the more likely it becomes for a hidden confounder to affect multiple of the these treatments rather than a single one of them.
	
	\begin{theorem} \label{th:sequential_strong_ignorability}
	If the distribution of the assigned causes $p(\hist{a}_T)$ can be written as the factor model $p(\theta_{1:k}, \hist{x}_T, \hist{z}_T,  \hist{a}_T)$, we obtain \textit{sequential ignorable treatment assignment}:
	\begin{equation}
	\vc{Y}(\bar{\vc{a}}_{\geq t})  \ci  (A_{t1}, \dots, A_{tk}) \mid \hist{A}_{t-1},  \hist{X}_{t}, \hist{Z}_{t},
	\end{equation}
	for all $\hist{a}_{\geq t}$ and for all $t\in \{0, \dots, T\}$. 
	\end{theorem}
	
	Theorem 1 is proved by leveraging Assumption \ref{eq:single_strong_ignorability}, the fact that the latent variables $\vc{Z}_t$ are inferred without knowledge of the potential outcomes $\vc{Y}(\hist{a}_{\geq t})$ and the fact that the causes $(A_{t1}, \dots, A_{tk})$ are jointly independent given $\vc{Z}_t$ and $\vc{X}_t$. The result means that, at each timestep, the variables $\hist{X}_{t}, \hist{Z}_{t}, \hist{A}_{t-1}$ contain all of the dependencies between the potential outcomes $\vc{Y}(\hist{a}_{\geq t})$ and the assigned causes $\vc{A}_{t}$.  See Appendix A for the full proof.
	
	As discussed in \citet{wang2018blessings}, the substitute confounders $\vc{Z}_t$ also need to satisfy positivity (Assumption 2), i.e. if $P(\hist{A}_{t-1} = \hist{a}_{t-1}, \hist{Z}_t = \hist{z}_t, \hist{X}_t = \hist{x}_t) \neq 0$ then $P(\vc{A}_t = \vc{a}_t \mid \hist{A}_{t-1} = \hist{a}_{t-1}, \hist{Z}_t = \hist{z}_t, \hist{X}_t = \hist{x}_t) > 0$ for all $\vc{a}_t$. After fitting the factor model, this can be tested \citep{robins2008estimation}.  When positivity is limited, the outcome model estimates of treatment responses will also have high variance. In practice, positivity can be enforced by setting the dimensionality of $\vc{Z}_t$ to be smaller than the number of treatments \citep{wang2018blessings}.

	\textbf{Predictive Checks over Time}: The theory holds if the fitted factor model captures well the distribution of assigned treatments. This condition can be assessed by extending predictive model checking \cite{rubin1984bayesianly}  to the time-series setting. We compute $p$-values over time to evaluate how similar the distribution of the treatments learned by the factor model is with the distribution of the treatments in a validation set of patients. 	At each timestep $t$, for the patients in the validation set, we obtain $M$ replicas of their treatment assignments $\{\vc{a}_{t, \text{rep}}^{(i)}\}_{i=1}^{M}$ by sampling from the factor model.  The replicated treatment assignments are compared with the actual treatment assignments, $\vc{a}_{t, \text{val}}$, using the test statistic $T(\vc{a}_t)$:
	\begin{equation}
	T(\vc{a}_t) = \bb{E}_Z[\text{log}\,\, p(\vc{a}_t \mid Z_t, X_t)],
	\end{equation}
	which is related to the marginal log likelihood \cite{wang2018blessings}. The predictive $p$-value for timestep $t$ is computed as follows: 
	\begin{equation}
	\dfrac{1}{M} \sum_{i=1}^{M} \mathbf{1} \left(T\left(\vc{a}_{t, \text{rep}}^{(i)}\right) < T\left(\vc{a}_{t, \text{val}}\right)\right),
	\end{equation}
	where $\mathbf{1}(\cdot)$ represents the indicator function. 
	
	If the model captures well the distribution of the assigned causes, then the test statistics for the treatment replicas are similar to the test statistics for the treatments in the validation set, which makes $0.5$ the ideal $p-$value in this case.

	\subsection{Outcome Model} \label{sec:outcome_model}
	
	If the factor model passes the predictive checks, the Time Series Deconfounder fits an outcome model \citep{robins2000marginal, lim2018forecasting}  to estimate individualized treatment effects over time. After sampling the sequence of latent variables $\hat{\hist{Z}}_t = (\hat{\vc{Z}}_1, \dots, \hat{\vc{Z}}_t)$ from the factor model, the outcome model can be used to estimate $\mathbb{E} [\vc{Y} \mid \hist{a}_{\geq t}, \hist{A}_{t-1}, \hist{X}_t, \hat{\hist{Z}}_t] = \mathbb{E} [\vc{Y}(\hist{a}_{\geq t}) \mid  \hist{A}_{t-1}, \hist{X}_t,  \hat{\hist{Z}}_t]$.

	To compute uncertainty estimates of the potential outcomes, we can sample $\hat{\hist{Z}}_t$ repeatedly and then fit an outcome model for each sample to obtain multiple point estimates of $\vc{Y}(\hist{a}_{\geq t})$. The variance of these point estimates will represent the uncertainty of the Time Series Deconfounder.
	
	\citet{d2019multi} raised some concerns about identifiability of the mean potential outcomes using the deconfounder framework in \citet{wang2018blessings} in the static setting and illustrated some pathological examples where identifiability might not hold.\footnote{See \citet{wang2018blessings} for a longer discussion addressing the concerns in \cite{d2019multi}.} In practical settings, the outcome estimates from the Time Series Deconfounder are identifiable, as supported by the experimental results in Sections \ref{sec:experiments_synthetic} and \ref{sec:experiments_mimic}. Nevertheless, when identifiability  represents an issue, the uncertainty in the potential outcomes can be used to assess the reliability of the Time Series Deconfounder. In particular, the variance in the potential outcomes indicates how the finite observational data inform the estimation of substitutes for the hidden confounders and subsequently the treatment outcomes of interest. When the treatment effects are non-identifiable, the estimates of the Time Series Deconfounder will have high variance. 
	
	By using this framework to estimate substitutes for the hidden confounders we are trading off confounding bias for estimation variance \citep{wang2018blessings}. The treatment effects computed without accounting for the hidden confounders will inevitably be biased. Alternatively, using the latent variables from the factor model will result in unbiased, but higher variance estimates of treatment effects. 
	
	\begin{figure*}[ht]
		\begin{center}
			\centerline{\includegraphics[width=1.99\columnwidth]{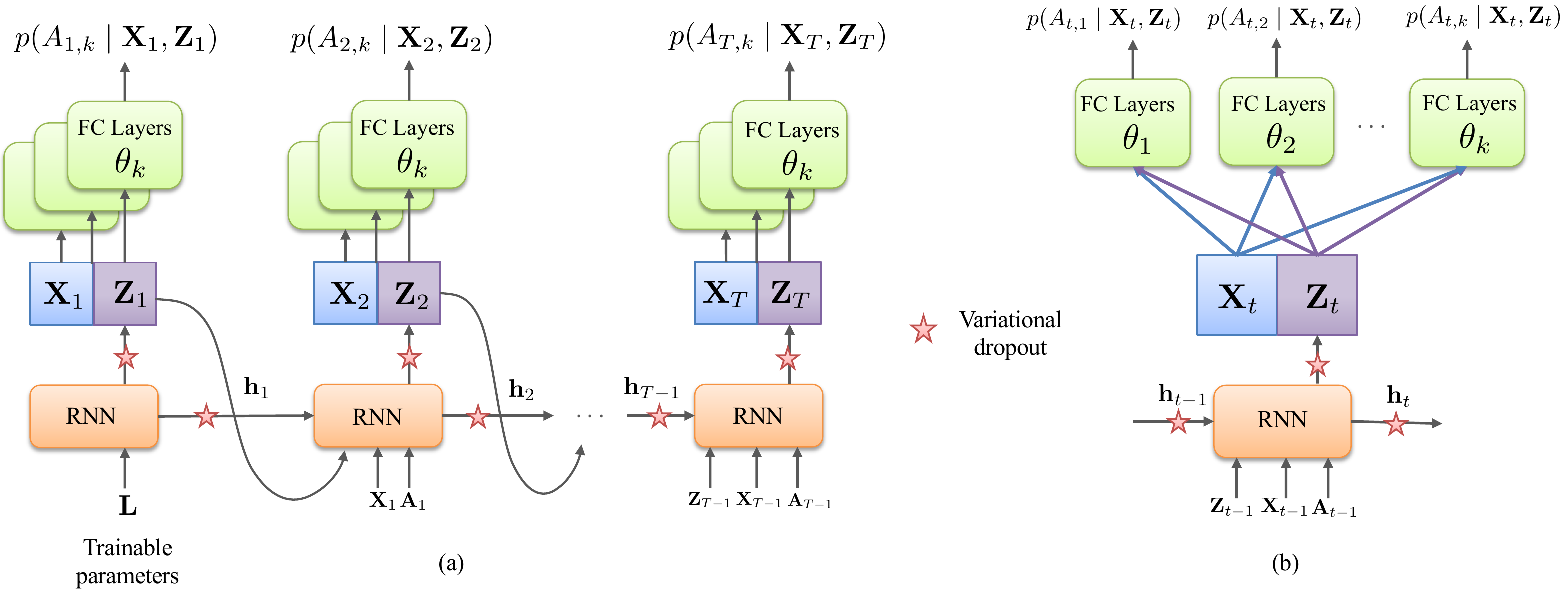}}
			\caption{(a) Proposed factor model implementation. $\vc{Z}_t$ is generated by RNN as a function of the history $\hist{H}_{t-1}$, given by the hidden state $\vc{h}_t$, and current input. Multitask output is used to construct the treatments such that they are independent given $\vc{Z}_t$ and $\vc{X}_t$. (b) Closer look at a single timestep.}
			\label{fig:factor_model}
		\end{center}
		\vspace{-0.8cm}
	\end{figure*}

	\section{Factor Model over Time in Practice} \label{sec:factor_model_implementation}
	
	Since we are dealing with time-varying treatments, we cannot use standard factor models, such as PCA \citep{tipping1999probabilistic} or Deep Exponential Families \citep{ranganath2015deep}, as they can only be applied in the static setting. Using the theory developed for the factor model over time we introduce a practical implementation based on a recurrent neural network (RNN) with multitask output and variational dropout as illustrated in Figure \ref{fig:factor_model}. 
	
	The recurrent part of the model infers the latent variables $\vc{Z}_t$ such that they depend on the patient history: 
	\begin{eqnarray}
    \vc{Z}_1 &=& \text{RNN}(\vc{L}), \\
    \vc{Z}_t &=& \text{RNN}(\hist{Z}_{t-1}, \hist{X}_{t-1}, \hist{A}_{t-1}, \vc{L}),
    \end{eqnarray}
    where $\vc{L}$ consists of randomly initialized  parameters that are trained with the rest of the parameters in the RNN. 
    
    The size of the RNN output is $D_Z$ and this specifies the size of the latent variables that are inferred as substitutes for the hidden confounders. In our experiments, we use an LSTM unit \citep{hochreiter1997long} as part of the RNN. Moreover, to infer the assigned treatments at timestep $t$, $\vc{A}_t = [A_{t1}, \dots, A_{tk}]$ such that they are conditionally independent given the latent variables $\vc{Z}_t$ and the observed covariates $\vc{X}_t$, we propose using multitask multilayer perceptrons (MLPs) consisting of fully connected (FC) layers: 
    \begin{equation}
    A_{tj} = \text{FC} (\vc{X}_t, \vc{Z}_t; \theta_j), 
    \end{equation} 
	for all $j = 1, \dots k$ and for all $t = 1, \dots T$, where $\theta_j$ are the parameters in the FC layers used to obtain $A_{tj}$. We use a single FC hidden layer before the output layer.  For binary treatments, the sigmoid activation is used in the output layer. For continuous treatments, MC dropout \citep{gal2016dropout} can instead be applied in the FC layers to obtain $p(A_{tj} \mid \vc{X}_j, \vc{Z}_j)$. 
	
	To model the probabilistic nature of factor models we incorporate \textit{variational dropout} \citep{gal2016theoretically} in the RNN as illustrated in Figure  \ref{fig:factor_model}. Using dropout enables us to obtain samples from $\vc{Z}_t$ and treatment assignments $A_{tj}$.  These samples allow us to obtain treatment replicas and to compute predictive checks over time, but also to estimate the \textit{uncertainty} in $\vc{Z}_t$ and potential outcomes. 
	
	Using the treatment assignments from the observational dataset, the factor model can be trained using gradient descent based methods. The proposed factor model architecture follows from the theory developed in Section \ref{sec:deconfounding} where at each timestep the latent variable $\vc{Z}_t$ is built as a function of the history (parametrized by an RNN). The multitask output is essential for modeling the conditional independence between the assigned treatments given the latent variables generated by the RNN and the observed covariates. The factor model can be extended to allow for irregularly sampled data by using a PhasedLSTM \cite{neil2016phased}.

	Note that our theory does not put restrictions on the factor model that can be used. Alternative factor models over time are generalized dynamic-factor model \cite{forni2000generalized, forni2005generalized} or factor-augmented vector autoregressive models \cite{bernanke2005measuring}. These come from the econometrics literature and explicitly model the dynamics in the data. The use of RNNs in the factor model enables us to learn complex relationships between $\hist{X}_t, \hist{Z}_t$, and $\hist{A}_t$ from the data, which is needed in medical applications involving complex diseases. Nevertheless, predictive checks should be used to assess any selected factor model.
	
	\section{Experiments on Synthetic Data} \label{sec:experiments_synthetic}
	
	To validate the theory developed in this paper, we perform experiments on synthetic data where we vary the effect of hidden confounding. It is not possible to validate the method on real datasets since the true extent of hidden confounding is never known \citep{wang2018blessings, louizos2017causal}.

	\subsection{Simulated Dataset} \label{sec:simulated_data}
	
	To keep the simulation process general, we propose building a dataset using $p$-order autoregressive processes. At each timestep $t$,  we simulate $k$ time-varying covariates $X_{t, k}$ representing single cause confounders and a multi-cause hidden confounder $Z_t$ as follows:
	\begin{eqnarray}
	X_{t, j} &=& \frac{1}{p}  \sum_{i=1}^p \left( \alpha_{i, j} X_{t-i, j} + \omega_{i, j} A_{t-i, j} \right)+ \eta_t \\ 
	Z_t &=& \frac{1}{p} \sum_{i=1}^{p}  ( \beta_i Z_{t-i} + \sum_{j=1}^k \lambda_{i,j} A_{t-i, j} ) + \epsilon_t,
	\end{eqnarray}
	
	for $j = 1, \dots, k$,  $\alpha_{i, k}, \lambda_{i,j} \sim  \cl{N}(0, 0.5^2)$, $\omega_{i, k}, \beta_i \sim \cl{N}(1 - (i/p), (1/p)^2)$, and $\eta_t, \epsilon_t \sim \cl{N}(0, 0.01^2)$. The value of $Z_t$ changes over time and is affected by the treatment assignments. 
	
	Each treatment assignment $A_{t,j}$ depends on the single-cause confounder $X_{t,j}$ and multi-cause hidden confounder $Z_t$:
	\begin{eqnarray}
	\vc{\pi}_{tj} &=& \gamma_A  \hat{Z}_{t} + (1-\gamma_A)  \hat{X}_{tj} \\
	A_{tj} \mid \vc{\pi}_{tj}  &\sim& \text{Bernoulli}(\sigma(\lambda\pi_{tj})), \,\, 
	\end{eqnarray}
	where $\hat{X}_{tj}$ and $\hat{Z}_{t}$ are the sum of the covariates and confounders respectively over the last $p$ timesteps, $\lambda = 15$, $\sigma(\cdot)$ is the sigmoid function and $\gamma_A$ controls the amount of hidden confounding applied to the treatment assignments. The outcome is also obtained as a function of the covariates and the hidden confounder:
	\begin{equation}
	\vc{Y}_{t+1} =\gamma_Y Z_{t+1} + (1-\gamma_Y) \Big(\dfrac{1}{k} \sum_{j=1}^k X_{t+1, j}\Big),
	\end{equation}
	where $\gamma_Y$ controls the amount of hidden confounding applied to the outcome. 	We simulate datasets consisting of 5000 patients, with trajectories between 20 and 30 timesteps, and $k=3$ covariates and treatments. To induce time dependencies we set $p=5$. Each dataset undergoes a 80/10/10 split for training, validation and testing respectively. Hyperparameter optimization is performed for each trained factor model as explained in Appendix B. Using the training observational dataset, we fit the Time Series Deconfounder to perform one-step ahead estimation of treatment responses.
	

	\subsection{Evaluating Factor Model using Predictive Checks}
	
	Our theory for using the inferred latent variables as substitutes for the hidden confounders and obtain unbiased treatment responses relies on the fact that the factor model captures well the distribution of the assigned causes.
	\begin{figure}[H]
		\vspace{-0.4cm}
		\begin{center}
		\includegraphics[width=\columnwidth]{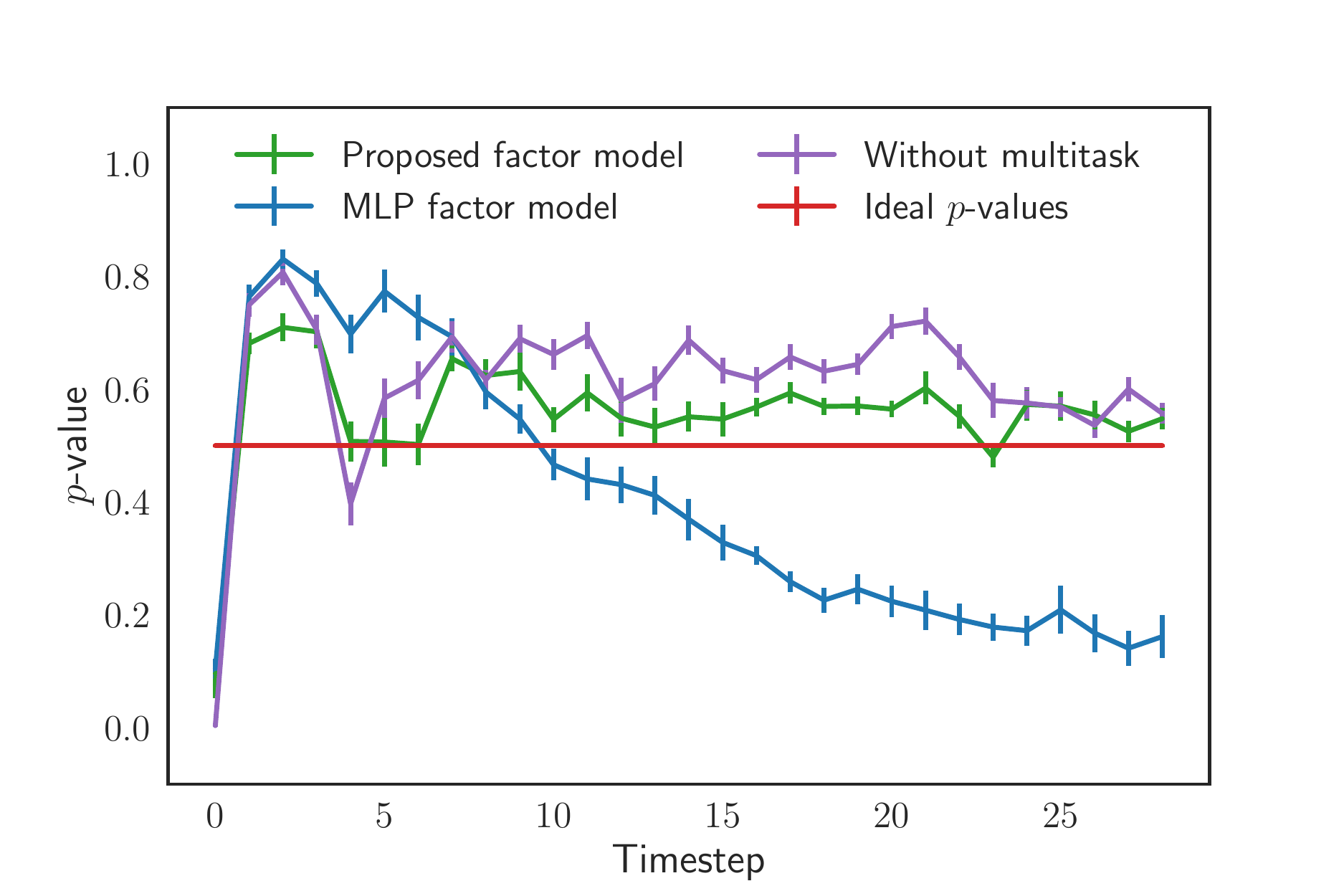}
		\end{center}
		\caption{Predictive checks over time. We show the mean $p$-values at each timestep and the std error. }
		\vspace{-0.37cm}
		\label{fig:predictive_checks}
	\end{figure}
	
		\begin{figure*}[t]
		\begin{center}
			\centerline{\includegraphics[width=2.0\columnwidth]{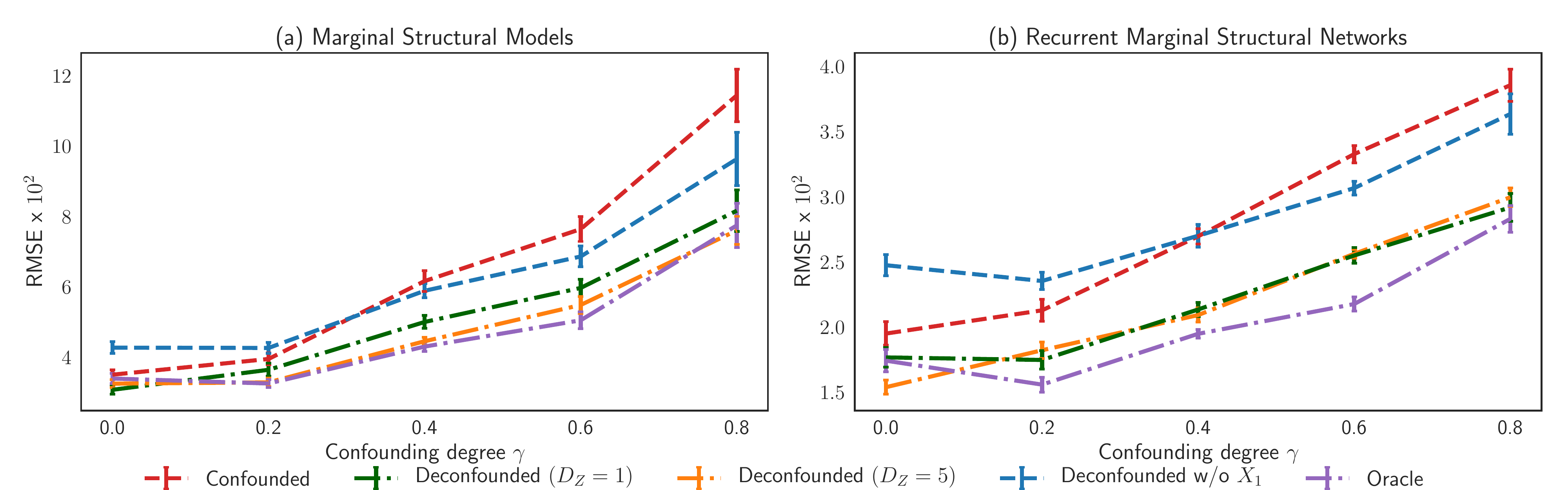}}
			\caption{Results for deconfounding the one-step ahead estimation of treatment responses in two outcome models: (a) Marginal Structural Models (MSM) and (b) Recurrent Marginal Structural Networks (R-MSN). The average RMSE and the standard error in the results are computed for 30 dataset simulations for each different degree of confounding, as measured by $\gamma$. }
			\label{fig:deconfounding_eval}
		\end{center}
		\vspace{-0.5cm}
	\end{figure*}

	To assess the suitability of our proposed factor model architecture, we compare it with the following baselines: RNN without multitask output (predicting the $k$ treatment assignments by passing $\vc{X}_t$ and $\vc{Z}_t$ through a FC layer and output layer with $k$ neurons) and multilayer perceptron (MLP) used instead of the RNN at each timestep to generate $\vc{Z}_t$. The MLP factor model does not use the entire history for generating $\vc{Z}_t$. See Appendix C for details. 
	
	Figure \ref{fig:predictive_checks} shows the $p$-values over time computed for the test set in 30 simulated datasets with $\gamma_A = \gamma_Y = 0.5$. The $p$-values for the MLP factor model decrease over time, which means that there is a consistent distribution mismatch between the treatment assignments learned by this model and the ones in the test set. Conversely, the predictive checks for our proposed factor model are closer to the ideal $p$-value of 0.5. This illustrates that having an architecture capable of capturing time-dependencies and accumulating past information for inferring the latent confounders is crucial. Moreover, the performance for the RNN without multitask is similar to our model, which indicates that the factor model constraint does not affect the performance in capturing the distribution of the causes. 
	
	\subsection{Deconfounding the Estimation of Treatment Responses over Time} \label{sec:experiments_deconfounding}
	
	We evaluate how well the Time Series Deconfounder  \footnote{The implementation of the Time Series Deconfounder can be found at \url{https://bitbucket.org/mvdschaar/mlforhealthlabpub/src/master/alg/time_series_deconfounder/} and at \url{https://github.com/ioanabica/Time-Series-Deconfounder}.} can remove hidden confounding bias when used in conjunction with the following outcome models:

	\textbf{Standard Marginal Structural Models (MSMs)}. MSMs \citep{robins2000marginal, hernan2001marginal} have been widely used in epidemiology to estimate treatment effects over time. MSMs use inverse probability of treatment weighting (IPTW) to adjust for the time-dependent confounding bias present in observational datasets \cite{mansournia2017handling, bica2020real}. MSMs compute the propensity weights by using logistic regression; through IPTW, these models construct a pseudo-population from the observational data where the treatment assignment probability no longer depends on the time-varying confounders. The treatment responses over time are computed using linear regression. For full implementation details in Appendix D.1. 

	\textbf{Recurrent Marginal Structural Networks (R-MSNs)}. R-MSNs \citep{lim2018forecasting}  also use IPTW to remove the bias from time-dependent confounders when estimating treatment effects over time. However, R-MSNs  estimate the propensity scores using RNNs instead. The use of RNNs is more robust to changes in the treatment assignment policy. To estimate the treatment responses over time R-MSNs also use a model based on RNNs. For implementation details, see Appendix D.2. 
	
	Notice that these outcome models were also chosen because they are capable of handling multiple treatments (that may be assigned simultaneously) at the same timstep. In the simulated dataset, parameters $\gamma_A$ and $\gamma_Y$ control the amount of hidden confounding applied to the treatments  and outcomes respectively. We vary this amount through $\gamma_A = \gamma_Y = \gamma $. 
	
	The Time Series Deconfounder is used to obtain unbiased estimates of one-step ahead treatment responses. For a comparative evaluation, the outcome models are trained to estimate these treatment responses in the following scenarios:  without information about $\hist{Z}_t$ in which case they estimate $\mathbb{E}[\vc{Y}_{t+1}(\vc{a}_t)\mid \hist{A}_{t-1}, \hist{X}_t]$ (Confounded), with information about the simulated (oracle) $\hist{Z}_t$ which leads to the following regression model $\mathbb{E}[\vc{Y}_{t+1}(\vc{a}_t)\mid \hist{A}_{t-1}, \hist{X}_t, \hist{Z}_t]$ (Oracle), as well as after applying the Time Series Deconfounder with different model specifications and in this case $\mathbb{E}[\vc{Y}_{t+1}(\vc{a}_t)\mid \hist{A}_{t-1}, \hist{X}_t, \hat{\hist{Z}}_t]$ is estimated (Deconfounded). To highlight the importance of Assumption 3, we also apply the Time Series Deconfounder after removing the single-cause confounder $X_1$, thus violating the assumption.  
	
	Figure \ref{fig:deconfounding_eval} shows the root mean squared error (RMSE) for the one-step ahead estimation of treatment responses for patients in the test set. We notice that the Time Series Deconfounder gives unbiased estimates of treatment responses, i.e. close to the estimates obtained using the simulated (oracle) confounders.  The method is robust to model misspecification, performing similarly when $D_Z= 1$ (simulated size of hidden confounders) and when $D_Z=5$ (misspecified size of inferred confounders). When there are no hidden confounders ($\gamma = 0$), the extra information from $\hat{\hist{Z}}_t$ does not harm the estimations (although they have higher variance). 
	
	When the sequential single strong ignorability assumption (Assumption 3) is invalidated, namely when the single cause confounder $X_1$ is removed from the observational dataset, we obtain biased estimates of the treatment responses. The performance in this case, however, is comparable to the performance when there is no control for the unobserved confounders.     
	
    In Appendix E, we consider an experimental set-up with a 
    different simulated size for the hidden confounders (true $D_Z = 3$) and show results when the size of the hidden confounders is underestimated in the Time Series Deconfounder. We also include 
    additional results on a simulated setting with static hidden confounders. 
	
	\textbf{Source of gain:}   To understand the source of gain in the Time Series Deconfounder, consider why the outcome models fail in the scenarios when there are hidden confounders. MSMs and R-MSNs make the implicit assumption that the treatment assignments depend only on the observed history.  The existence of any multi-cause confounders not captured by the history results in biased estimates of both the propensity weights and of the outcomes. On the other hand, the construction in our factor model rules out the existence of any multi-cause confounders which are not captured by $\vc{Z}_t$. By augmenting the data available to the outcome models with the substitute confounders, we eliminate these biases.

		\begin{table*}[ht]
		\caption{Average  RMSE $\times 10^2$  and standard error in the results for predicting the effect of antibiotics, vassopressors and mechanical ventilator on three patient covariates. The results are for 10 runs.}
		\label{tab:mimic}
		\begin{center}
			\begin{small}
				\begin{tabular}{c|cc|cc|cc}
					\toprule
					& \multicolumn{2}{c|}{White blood cell count} & \multicolumn{2}{c|}{Blood pressure} & \multicolumn{2}{c}{Oxygen saturation} \\
					Outcome model  & MSM & R-MSN & MSM & R-MSN & MSM & R-MSN  \\
					\hline
					Confounded  & $3.90 \pm 0.00$  &  $2.91 \pm 0.05$ & $12.04 \pm 0.00$&  $10.29 \pm 0.05$ & $2.92 \pm 0.00$  & $1.74 \pm 0.03$ \\
					\hline
					Deconfounded ($D_Z = 1$) & $3.55 \pm 0.05$  & $2.62 \pm 0.07$ & $11.69 \pm 0.14$  &  $9.35 \pm 0.11$ &  $2.42 \pm 0.02$  &  $1.24 \pm 0.05$ \\
					Deconfounded ($D_Z = 5$)  &  $3.56 \pm 0.04$& $2.41 \pm 0.04$ & $11.63 \pm 0.10$ &  $9.45 \pm 0.10$  &$2.43 \pm 0.02$  &  $1.21 \pm 0.07$  \\
					Deconfounded ($D_Z = 10$)  & $3.58 \pm 0.03$  & $2.48 \pm 0.06$ & $11.66 \pm 0.14$  &  $9.20 \pm 0.12$ &  $2.42 \pm 0.01$ & $1.17 \pm 0.06$  \\
					Deconfounded ($D_Z = 20$)  & $3.54 \pm 0.04$ & $2.55 \pm 0.05$ & $11.57 \pm 0.12$  &  $9.63 \pm 0.14$ & $2.40\pm 0.01$ & $1.28 \pm 0.08$\\
					\bottomrule
				\end{tabular}
			\end{small}
		\end{center}
		\vspace{-0.2cm}
	\end{table*}
	
	\section{Experiments on MIMIC III} \label{sec:experiments_mimic}
	
    Using the Medical Information Mart for Intensive Care (MIMIC III) \citep{johnson2016mimic} database consisting of electronic health records from patients in the ICU, we show how the Time Series Deconfounder can be applied on a real dataset. From MIMIC III we extracted a dataset with 6256 patients for which there are three treatment options at each timestep: antibiotics, vasopressors, and mechanical ventilator (all of which can be applied simultaneously). These treatments are common in the ICU and are often used to treat patients with sepsis \citep{schmidt2016evaluation, scheeren2019current}.  For each patient, we extracted $25$ patient covariates consisting of lab tests and vital signs measured over time that affect the assignment of treatments. We used daily aggregates of the patient covariates and treatments and patient trajectories of up to $50$ timesteps. We estimate the effects of antibiotics, vasopressors, and mechanical ventilator on the following patient covariates: white blood cell count, blood pressure, and oxygen saturation. 
    
    Hidden confounding is present in the dataset as patient comorbidities and several lab tests were not included. However, since this is a real dataset, it is not possible to evaluate the extent of hidden confounding or to estimate the true (Oracle) treatment responses.
	
    Table \ref{tab:mimic} illustrates the RMSE when estimating one-step ahead treatment responses by using the MSM and R-MSN outcome models directly on the extracted dataset (Confounded) and after applying the Time Series Deconfounder and augmenting the dataset with the substitutes for the hidden confounders of different dimensionality $D_Z$ (Deconfounded). We notice that in all cases, the Time Series Deconfounder enables us to obtain a lower error when estimating the effect of antibiotics, vasopressors, and mechanical ventilator on the patients' white blood cell count, blood pressure, and oxygen saturation. By modeling the dependencies in the assigned treatments for each patient, the factor model part of the Time Series Deconfounder was able to infer latent variables that account for the unobserved information about the patient states. Using these substitutes for the hidden confounders in the outcome models resulted in better estimates of the treatment responses. While these results on real data require further validation from doctors (which is outside the scope of this paper), they indicate the potential of the method to be applied in real medical scenarios. 
	
    In Appendix E, we include results for an additional experimental set-up where we remove several patient covariates from the dataset and we show how the Time Series Deconfounder can be used to account for this bias. Moreover, in Appendix F, we provide further discussion and directions for future work. 

	\section{Conclusion}
	
	The availability of observational data consisting of longitudinal information about patients prompted the development of methods for modeling the effects of treatments on the disease progression in patients. All existing methods for estimating the individualized effects of time-dependent treatment from observational data make the untestable assumption that there are no hidden confounders. In the longitudinal setting, this assumption is even more problematic than in the static setting. As the state of the patient changes over time and the complexity of the treatment assignments and responses increases, it becomes much easier to miss important confounding information.  
    
    In this paper, we proposed the Time Series Deconfounder, a method that takes advantage of the patterns in the multiple treatment assignments over time to infer latent variables that can be used as substitutes for the hidden confounders. Moreover, we developed a deep learning architecture based on an RNN with multitask output and variational dropout for building a factor model over time and computing the latent variables in practice. Through experimental results on both synthetic and real datasets, we show the effectiveness of the Time Series Deconfounder in removing the bias from the estimation of treatment responses over time in the presence of multi-cause hidden confounders. 
	
    \section*{Acknowledgements}
    The authors would like to thank the reviewers for their helpful feedback. The research presented in this paper was supported by The Alan Turing Institute, under the EPSRC grant EP/N510129/1 and by the US Office of Naval Research (ONR).

	\bibliography{refs}
	\bibliographystyle{icml2020}
	
	\appendix
    \onecolumn
    
   	\section{Proof for Theorem 1} \label{apx:proof_th_1}
	
	Before proving Theorem 1, we introduce several definitions and lemmas that will aid with the proof. Note that the these are extended from the static setting in \citet{wang2018blessings}. Remember that at each timestep $t$, the random variable $\vc{Z}_t \in \cl{Z}_t$ is constructed as a function of the history until timestep $t$: $\vc{Z}_t = g(\hist{H}_{t-1})$, where $\hist{H}_{t-1} = (\hist{Z}_{t-1}, \hist{X}_{t-1}, \hist{A}_{t-1})$ takes values in $\bar{\cl{H}}_{t-1} = \bar{\cl{Z}}_{t-1} \times \bar{\cl{X}}_{t-1} \times \bar{\cl{A}}_{t-1}$ and $g: \histcl{H}_{t-1} \rightarrow \cl{Z}$. In order to obtain \textbf{sequential ignorable treatment assignment} using the substitutes for the hidden confounders $\vc{Z}_t$, the following property needs to hold:
	\begin{equation}
	\vc{Y}(\bar{\vc{a}}_{\geq t})  \ci (A_{t1}, \dots, A_{tk}) \mid \bar{\vc{X}}_{t},  \bar{\vc{A}}_{t-1}, \bar{\vc{Z}}_{t},
	\end{equation}
	$ \forall \hist{a}_{\geq t}$ and $\forall t\in \{0, \dots, T\}$.
	
	\begin{definition}
	\textbf{Sequential Kallenberg construction} \\
	At timestep $t$, we say that the distribution of assigned causes $(A_{t1}, \dots A_{tk})$ admits a sequential Kallenberg construction from the random variables  $\vc{Z}_t = g(\hist{H}_{t-1})$ and $\vc{X}_t$ if there exist measurable functions $f_{tj}: \cl{Z}_t \times \cl{X}_t \times [0, 1] \rightarrow \cl{A}_j$ and random variables $U_{jt}\in [0, 1]$, with $ j = 1, \dots, k$ such that:
	\begin{equation}
	A_{tj} = f_{tj}(\vc{Z}_t, \vc{X}_{t}, U_{tj}),
	\end{equation}
	where $U_{tj}$ marginally follow Uniform$[0, 1]$ and jointly satisfy: 
	\begin{equation}
	(U_{t1}, \dots U_{tk}) \ci  \vc{Y}(\hist{a}_{\geq t}) \mid \vc{Z}_t, \vc{X}_{t}, \hist{H}_{t-1}, 
	\end{equation}
	for all $\hist{a}_{\geq t}$. 
	\end{definition}
	
	\begin{lemma}
	\textbf{Sequential Kallenberg construction at each timestep $t$ $\Rightarrow$ Sequential strong ignorability}. If at every timestep $t$, the distribution of assigned causes $(A_{t1}, \dots A_{tk})$ admits a Kallenberg construction from $\vc{Z}_t = g(\hist{H}_{t-1})$ and $\vc{X}_t$ then we obtain sequential strong ignorability. 
	\end{lemma}
	
	\begin{proof}
	Assume that $\cl{A}_j$ for $ j=1,\dots, k$ are Borel spaces. For any $t \in \{1, \dots, T\}$  assume $\cl{Z}_t$ and $\cl{X}_t$ are measurable spaces and assume that $A_{tj} = f_{tj}(\vc{Z}_t, \vc{X}_{t}, U_{tj})$, where $f_{tj}$ are measurable and 
	\begin{equation}
	(U_{t1}, \dots U_{tk}) \ci  \vc{Y}(\hist{a}_{\geq t}) \mid \vc{Z}_t, \vc{X}_{t}, \hist{H}_{t-1},
	\end{equation} 
	for all $\hist{a}_{\geq t}$. This implies that:
	\begin{equation}
	(\vc{Z}_t, \vc{X}_{t}, U_{t1}, \dots U_{tk}) \ci \vc{Y}(\hist{a}_{\geq t}) \mid \vc{Z}_t, \vc{X}_{t}, \hist{H}_{t-1}.
	\end{equation} 
	Since the $A_{tj}$'s are measurable functions of $(\vc{Z}_t, \vc{X}_{t}, U_{t1}, \dots U_{tk})$ and $\hist{H}_{t-1} = (\hist{Z}_{t-1}, \hist{X}_{t-1}, \hist{A}_{t-1})$,  we have that sequential strong ignorability holds:
	\begin{equation}
	(A_{t1}, \dots A_{tk}) \ci \vc{Y}(\hist{a}_{\geq t}) \mid \hist{X}_t, \hist{A}_{t-1}, \hist{Z}_t,
	\end{equation} 
	$ \forall \hist{a}_{\geq t}$ and $\forall t\in \{0, \dots, T\}$.
	\end{proof}
	
	\begin{lemma}
	\textbf{Factor models for the assigned causes $\Rightarrow$ Sequential Kallenberg construction at each timestep $t$}. Under weak regularity conditions, if the  distribution of assigned causes $p(\hist{a}_T)$ can be written as the factor model $p(\theta_{1:k}, \hist{x}_T, \hist{z}_T,  \hist{a}_T)$ then we obtain  a sequential Kallenberg construction for each timestep.
	
	Regularity condition: The domains of the causes $\cl{A}_j$ for $j=1, \dots, k$ are Borel subsets of compact intervals. Without loss of generality, assume $\cl{A}_j = [0, 1]$ for $j=1, \dots, k$.
	\end{lemma}
	
	The proof for Lemma 2 uses Lemma 2.22 in \citet{kallenberg2006foundations} (kernels and randomization): Let $\mu$ be a probability kernel from a measurable space $S$ to a Borel space $T$. Then there exists some measurable function $f: S \times [0, 1] \rightarrow T$ such that if $\vartheta$ is $U(0, 1)$, then $f(s, \vartheta)$ has distribution $\mu(s, \dot)$ for every $s\in S$. 
	\begin{proof}
	For timestep $t$, consider the random variables $A_{t1} \in \cl{A}_1, \dots A_{tk}\in \cl{A}_k, \vc{X}_t\in \cl{X}_t, \vc{Z}_t = g(\hist{H}_{t-1}) \in \cl{Z}_t$ and $\theta_j \in \Theta$. Assume sequential single strong ignorability holds.
	Without loss of generality, assume $\cl{A}_j = [0, 1]$ for $j=1, \dots, k$. 
	
	From Lemma 2.22 in Kallenberg (1997), there exists some measurable function $f_{tj}: \cl{Z}_t \times \cl{X}_t \times [0, 1] \rightarrow [0, 1]$ such that $U_{tj} \sim \text{Uniform}[0, 1]$ and:
	\begin{equation}
	A_{tj} = f_{tj} (\vc{Z}_t, \vc{X}_t, U_{tj})
	\end{equation}
	and there exists some measurable function $h_{tj}: \Theta \times [0, 1] \rightarrow [0, 1]$ such that: 
	\begin{equation}
	U_{tj} = h_{tj}(\theta_j, \omega_{tj}),
	\end{equation}
	where $\omega_{tj} \sim \text{Uniform}[0, 1]$ and $j=1, \dots, k$. 
	
	From our definition of the factor model we have that $\omega_{tj}$ for  $j=1, \dots, k$ are jointly independent. Otherwise, $A_{tj} = f_{tj} (\vc{Z}_t, \vc{X}_t, h_{tj}(\theta_j, \omega_{tj}))$ would not have been conditionally independent given $\vc{Z}_t, \vc{X}_t$. 
	
	Since sequential single strong ignorability holds at each timestep $t$, we have that $A_{tj}  \ci \vc{Y}(\bar{\vc{a}}_{\geq t})  \mid \vc{X}_{t}, \bar{\vc{H}}_{t-1} \,\,\, \forall \hist{a}\in \histcl{A}$, $\forall t\in \{0, \dots, T\}$ and for $j=1, \dots, k$ which implies:
	\begin{equation} \label{eq:omega_single_independence}
	\omega_{tj} \ci \vc{Y}(\bar{\vc{a}}_{\geq t}) \mid \vc{X}_{t}, \bar{\vc{H}}_{t-1},
	\end{equation}
	$\forall \hist{a}_{\geq t}\text{ and } \forall j \in \{1, \dots, k\}$. Using this, we can write:
	\begin{eqnarray}
	p(\vc{Y}(\hist{a}_{\geq t}), \omega_{t1}, \dots, \omega_{tk} \mid\vc{X}_{t}, \bar{\vc{H}}_{t-1}) &=& p(\vc{Y}(\hist{a}_{\geq t}) \mid \vc{X}_{t},\bar{\vc{H}}_{t-1}) \cdot  p(\omega_{t1}, \dots, \omega_{tk} \mid \vc{Y}(\hist{a}_{\geq t}), \vc{X}_{t}, \bar{\vc{H}}_{t-1}) \nonumber \\
	&=&  p(\vc{Y}(\hist{a}_{\geq t}) \mid \vc{X}_{t},\bar{\vc{H}}_{t-1}) \cdot \nonumber \prod_{j=1}^k p(\omega_{tj} \mid \omega_{t1}, \dots, \omega_{t, j-1},  \vc{Y}(\hist{a}_{\geq t}), \vc{X}_{t}, \bar{\vc{H}}_{t-1}) \nonumber \\
	&=& p(\vc{Y}(\hist{a}_{\geq t}) \mid \vc{X}_{t}, \bar{\vc{H}}_{t-1}) \cdot \prod_{j=1}^k p(\omega_{tj} \mid \vc{X}_{t}, \bar{\vc{H}}_{t-1}) \nonumber \\
	&=& p(\vc{Y}(\hist{a}_{\geq t}) \mid \vc{X}_{t}, \bar{\vc{H}}_{t-1}) \cdot \nonumber  p(\omega_{t1}, \dots, \omega_{tk} \mid \vc{X}_{t}, \bar{\vc{H}}_{t-1})
	\end{eqnarray}
	where the second and third steps follow form equation (\ref{eq:omega_single_independence}) and the fact that $\omega_{t1}, \dots, \omega_{tk}$ are jointly independent. This gives us:
	\begin{equation}
	(\omega_{t1}, \dots, \omega_{tk}) \ci \vc{Y}(\hist{a}_{\geq t}) \mid \vc{X}_{t}, \bar{\vc{H}}_{t-1}
	\end{equation}

	Moreover, since the latent random variable $\vc{Z}_t$ is constructed without knowledge of $\vc{Y}(\hist{a}_{\geq t})$, but rather as a function of the history $\hist{H}_{t-1}$ we have:
	\begin{equation}
	(\omega_{t1}, \dots, \omega_{tk}) \ci \vc{Y}(\hist{a}_{\geq t}) \mid \vc{Z}_t, \vc{X}_{t}, \bar{\vc{H}}_{t-1}.
	\end{equation}
	
	$\theta_{1:k}$ are parameters in the factor model and can be considered point masses, so we also have that:
	\begin{equation}
	(\theta_1, \dots, \theta_k) \ci  \vc{Y}(\hist{a}_{\geq t}) \mid \vc{Z}_t, \vc{X}_t, \hist{H}_{t-1},
	\end{equation} 
	
	Since $U_{tj}= (h_{ij}(\theta_j, \omega_{tj}))$ are measurable functions of $\theta_j$ and $\omega_{tj}$ we have that:
	\begin{equation}
	(U_{t1}, \dots, U_{tk}) \ci  \vc{Y}(\vc{a}_{\geq t}) \mid  \vc{Z}_t, \vc{X}_t, \hist{H}_{t-1}
	\end{equation}
	
	We have thus obtained a sequential Kallenberg construction at timestep $t$. 
	\end{proof}
	
	\begin{theorem}
	If the distribution of the assigned causes $p(\hist{a}_{T})$ can be written as the factor model $p(\theta_{1:k}, \hist{x}_T, \hist{z}_T,  \hist{a}_T)$ then we obtain \textit{sequential ignorable treatment assignment}:
	\begin{equation}
	\vc{Y}(\bar{\vc{a}}_{\geq t})  \ci (A_{t1}, \dots, A_{tk}) \mid \bar{\vc{X}}_{t}, \bar{\vc{Z}}_{t}, \bar{\vc{A}}_{t-1},  
	\end{equation}
	for all $\hist{a}_{\geq t}$ and for all $t\in \{0, \dots, T\}$. 
	\end{theorem}
	
	\begin{proof}
	Theorem 1 follows from Lemmas 1 and 2. In particular, using the proposed factor graph, we can obtain a sequential Kallenberg construction at each timestep and then obtain sequential strong ignorability. 
	\end{proof}

	\section{Implementation Details for the Factor Model} \label{apx:hyperparam}
	
	The factor model described in Section 5 was implemented in Tensorflow \citep{tensorflow2015-whitepaper} and trained on an NVIDIA Tesla K80 GPU. For each synthetic dataset (simulated as described in Section 6.1), we obtained 5000 patients, out of which 4000 were used for training, 500 for validation, and 500 for testing. Using the validation set, we perform hyperparameter optimization using 30 iterations of random search to find the optimal values for the learning rate, minibatch size (M), RNN hidden units, multitask FC hidden units and RNN dropout probability. LSTM \citep{hochreiter1997long} units are used for the RNN implementation. The search range for each hyperparameter is described in Table \ref{tab:hyperparameters_factor_model}. 
    
    The trajectories for the patients do not necessarily have to be equal. However, to be able to train the factor model, we zero-padded them such that they all had the same length. The patient trajectories were then grouped into minibatches of size M and the factor model was trained using the Adam optimizer \cite{kingma2014adam} for 100 epochs. 
	
	\begin{table}[ht]
		\caption{Hyperparameter search range for the proposed factor model implemented using a recurrent neural network with multitask output and variational dropout.}
		\label{tab:hyperparameters_factor_model}
		\vskip 0.15in
		\begin{center}
			\begin{small}
				\begin{tabular}{cc}
					\toprule
					Hyperparameter & Search range \\
					\hline
					\hline
					Learning rate &    0.01, 0.001, 0.0001    \\
					Minibatch size &  64, 128, 256  \\
					RNN hidden units &  32, 64, 128, 256  \\
					Multitask FC hidden units & 32, 64, 128\\
					RNN dropout probability &  0.1, 0.2, 0.3, 0.4, 0.5 \\
					\bottomrule
				\end{tabular}
			\end{small}
		\end{center}
		\vskip -0.1in
	\end{table}
	
	Table \ref{tab:hyperparameters_confounding} illustrates the optimal hyperparameters obtained for the factor model under the different amounts of hidden confounding applied (as described by the experiments in Section 6.1). Since the results for assessing the Time Series Deconfounder are averaged across  30 different simulated datasets, we report here the optimal hyperparameters identified through majority voting. We note that when the effect of the hidden confounders on the treatment assignments and the outcome is large, more capacity is needed in the factor model to be able to infer them.
	
	\begin{table*}[ht]
		\caption{Optimal hyperparameters for the factor model when different amounts of hidden confounding are applied in the synthetic dataset. The parameter $\gamma$ measures the amount of hidden confounding applied.  }
		\label{tab:hyperparameters_confounding}
		\vskip 0.15in
		\begin{center}
			\begin{small}
				\begin{tabular}{cccccc}
					\toprule
					Hyperparameter & $\gamma=0$ & $\gamma=0.2$ & $\gamma=0.4$  & $\gamma=0.6$ & $\gamma=0.8$\\
					\hline
					\hline
					Learning rate &    0.01  & 0.01 & 0.01 & 0.01& 0.001\\
					Minibatch size &  64 & 64 & 64 & 64 & 128 \\
					RNN hidden units & 32 & 64 & 64 & 128 & 128\\
					Multitask FC hidden units & 64 & 128 & 64 & 128 & 128\\
					RNN dropout probability &  0.2 & 0.2 &  0.1 & 0.3 & 0.3\\
					\bottomrule
				\end{tabular}
			\end{small}
		\end{center}
		\vskip -0.1in
	\end{table*}

	\begin{figure*}[ht]
		\vskip 0.2in
		\begin{center}
			\centerline{\includegraphics[width=1\columnwidth]{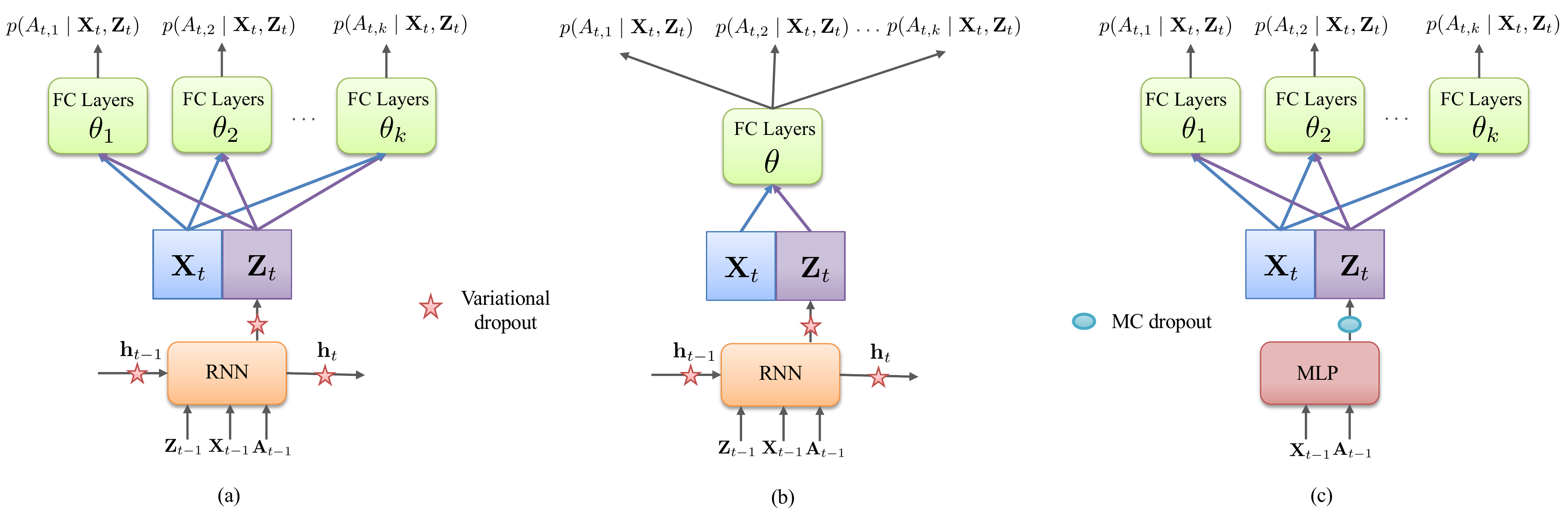}}
			\caption{(a) Proposed factor model using a recurrent neural network with multitask output and variational dropout. (b) Alternative design without multitask output. (c) Factor model using an MLP (shared across timestep) and multitask output. This baseline does not capture time-dependencies. MC dropout \cite{gal2016dropout} is applied in the MLP to be able to sample from the substitutes for the hidden confounders.}
			\label{fig:factor_model_baseline}
		\end{center}
		\vskip -0.2in
	\end{figure*} 
	
	\newpage
	\section{Baselines for Evaluating Factor Model} \label{apx:baseline_factor_model}

    Figure \ref{fig:factor_model_baseline} illustrates the architecture at each timestep for our proposed factor model and the baselines used for comparison. Figure \ref{fig:factor_model_baseline}(a) represents our proposed architecture for the factor model consisting of a recurrent neural network with multitask output and variational dropout. We want to ensure that the multitask constraint does not cause a decrease in the capability of the network to capture the distribution of the assigned causes. To do so, we compare our proposed factor model with the network in Figure \ref{fig:factor_model_baseline}(b) where we predict the $k$ treatment assignments by passing $\vc{X}_t$ and $\vc{Z}_t$ through a hidden layer and having an output layer with $k$ neurons. Moreover, to highlight the importance of learning time-dependencies to estimate the substitutes for the hidden confounders, we also use as a baseline the factor model in Figure \ref{fig:factor_model_baseline}(c). In this case, a multilayer perceptron (MLP) is shared across the timesteps and it infers the latent variable $Z_t$ using only the previous covariates and treatments. Note that in this case there is no dependency on the entire history.
	
	The baselines were optimised under the same set-up described for our proposed factor model in Appendix \ref{apx:hyperparam}. Tables \ref{tab:hyperparameters_multitaks} and \ref{tab:hyperparameters_mlp} describe the search ranges used for the hyperparameters in each of the baselines. 
	
	\begin{table}[ht]
		\caption{Hyperparameter search range for factor model without multitask (Figure \ref{fig:factor_model_baseline}(b)).}
		\label{tab:hyperparameters_multitaks}
		\vskip 0.15in
		\begin{center}
			\begin{small}
				\begin{tabular}{cc}
					\toprule
					Hyperparameter & Search range \\
					\hline
					\hline
					Learning rate &    0.01, 0.001, 0.0001    \\
					Minibatch size &  64, 128, 256  \\
					Max gradient norm & 1.0, 2.0, 4.0  \\
					RNN hidden units &  32, 64, 128, 256  \\
					Multitask FC hidden units & 32, 64, 128\\
					RNN dropout probability &  0.1, 0.2, 0.3, 0.4, 0.5 \\
					\bottomrule
				\end{tabular}
			\end{small}
		\end{center}
		\vskip -0.1in
	\end{table}
	
	\begin{table}[ht]
		\caption{Hyperparameter search range for MLP factor model. Figure \ref{fig:factor_model_baseline}(c))}
		\label{tab:hyperparameters_mlp}
		\vskip 0.15in
		\begin{center}
			\begin{small}
				\begin{tabular}{cc}
					\toprule
					Hyperparameter & Search range \\
					\hline
					\hline
					Learning rate &    0.01, 0.001, 0.0001    \\
					Minibatch size &  64, 128, 256  \\
					MLP hidden layer size &  32, 64, 128, 256  \\
					Multitask FC hidden units & 32, 64, 128\\
					MLP dropout probability &  0.1, 0.2, 0.3, 0.4, 0.5 \\
					\bottomrule
				\end{tabular}
			\end{small}
		\end{center}
		\vskip -0.1in
	\end{table}

	\section{Outcome Models}
	
	After inferring the substitutes for the hidden confounders using the factor model, we implement outcome models to estimate the individualised treatment responses:
	\begin{equation}
	\mathbb{E} [\vc{Y}_{t+1}(\vc{a}_t) \mid \hist{A}_{t-1}, \hist{X}_t, \hist{Z}_t] = h(\hist{A}_t, \hist{X}_t, \hist{Z}_t)
	\end{equation}
	We train the outcome models and evaluate them on predicting the treatment responses for each timestep, i.e. one-step-ahead predictions, for the patients in the test set. For training and tuning the outcome models, we use the same train/validation/test splits that we have used for the factor model. This means that the substitutes for the hidden confounders estimated using the fitted factor model on the test set are also used for testing purposes in the outcome models.
	
	\subsection{Marginal Structural Models} \label{apx:marginal_structural_models}
	
	MSMs \citep{robins2000marginal, hernan2001marginal} have been widely used in epidemiology to perform causal inference in longitudinal data. MSMs use inverse probability of treatment weighting during training to construct a pseudo-population from the observational data that resembles the one in a clinical trial and thus remove the bias introduced by time-dependent confounders \citep{platt2009time}. The propensity scores for each timestep are computed as follows: 
	\begin{eqnarray}
	\text{SW}_t = \dfrac{f(\vc{A}_t \mid \hist{A}_{t-1})}{ f(\vc{A}_t \mid \hist{X}_{t}, \hist{Z}_t, \hist{A}_{t-1})} =  \dfrac{\prod_{j=1}^{k} f(A_{t,j} \mid \hist{A}_{t-1})}{\prod_{j=1}^{k} f(A_{t,j} \mid \hist{X}_{t}, \hist{Z}_t, \hist{A}_{t-1})}
	\end{eqnarray}
	where $f(\cdot)$ is the conditional probability mass function for discrete treatments and the conditional probability density function for continuous treatments. We adopt the implementation in \citet{hernan2001marginal, howe2012estimating} for MSMs and estimate the propensity weights using logistic regression as follows: 
	\begin{equation}
	f(A_{t,k} \mid \hist{A}_{t-1}) = \sigma \Big(\sum_{j=1}^k \omega_{k}  (\sum_{i=1}^{t-1} A_{t,j}) \Big)
	\end{equation}
	\begin{equation}
	f(A_{t,k} \mid  \hist{X}_{t}, \hist{Z}_t, \hist{A}_{t-1}) = \sigma \Big(\sum_{j=1}^k \phi_{k}  (\sum_{i=1}^{t-1} A_{t,j})   
	+ \vc{w}_1 \vc{X}_t + \vc{w}_2 \vc{X}_{t-1} + \vc{w}_3 \vc{Z}_t + \vc{w}_4 \vc{Z}_{t-1}   \Big)
	\end{equation}
	where $\omega_{\star}, \phi_{\star}$ and $\vc{w}_{\star}$ are regression coefficients and $\sigma(\cdot)$ is the sigmoid function. 
	
	For predicting the outcome, the following regression model is used, where each individual patient is weighted by its propensity score:
	\begin{equation}
	h(\hist{A}_t, \hist{X}_t, \hist{Z}_t) = \sum_{j=1}^k \beta_{k}  (\sum_{i=1}^{t} A_{t,j}) + \vc{l}_1 \vc{X}_t  
	+ \vc{l}_2 \vc{X}_{t-1} + \vc{l}_3 \vc{Z}_t + \vc{l}_4 \vc{Z}_{t-1}   
	\end{equation}
	where $\beta_{\star}$ and $\vc{l}_{\star}$ are regression coefficients. Since MSMs do not require hyperparameter tuning, we train them on the patients from both the train and validation sets.

	\subsection {Recurrent Marginal Structural Networks} \label{apx:recurrent_marginal_structural_networks}
	
	R-MSNs, implemented as descried in \citet{lim2018forecasting}\footnote{We used the publicly available immlementation from \url{https://github.com/sjblim/rmsn_nips_2018.}}, use  recurrent neural networks to estimate the propensity scores and to build the outcome model. The use of RNNs is more robust to changes in the treatment assignment policy. Moreover, R-MSNs represent the first application of deep learning in predicting time-dependent treatment effects. The propensity weights are estimated using recurrent neural networks as follows:
	\begin{align}
	f(A_{t,k} \mid \hist{A}_{t-1}) = \text{RNN}_1(\hist{A}_{t-1}) && f(A_{t,k} \mid  \hist{X}_{t}, \hist{Z}_t, \hist{A}_{t-1}) = \text{RNN}_2(\hist{X}_{t}, \hist{Z}_t, \hist{A}_{t-1})
	\end{align}
	
	For predicting the outcome, the following prediction network is used:
	\begin{eqnarray}
	h(\hist{A}_t, \hist{X}_t, \hist{Z}_t) =  \text{RNN}_3(\hist{X}_{t}, \hist{Z}_t, \hist{A}_{t}),
	\end{eqnarray}
	where in the loss function, each patient is weighted by its propensity score. Since the purpose of our method is not to improve predictions, but rather to assess how well the R-MSNs can be deconfounded using our method, we use the optimal hyperparameters for this model, as identified by \citet{lim2018forecasting}. R-MSNs are then trained on the combined set of patients from the training and validation sets. 
	
	\begin{table}[h]
		\caption{Hyperparameters used for R-MSN.}
		\label{tab:hyperparameters_rmsn}
		\vskip 0.15in
		\begin{center}
			\begin{small}
				\begin{tabular}{c|ccc}
					\toprule
					Hyperparameter & \multicolumn{2}{c}{Propensity networks}  & \multirow{2}{1.1cm}{Prediction network} \\
					& $f(\vc{A}_t\mid \hist{A}_{t-1})$ & $f(\vc{A}_t\mid \hist{H}_{t})$ &   \\
					\hline
					Dropout rate &    0.1 & 0.1 & 0.1    \\
					State size & 6 & 16 & 16  \\
					Minibatch size & 128 & 64 & 128  \\
					Learning rate & 0.01 & 0.01 & 0.01 \\
					Max norm & 2.0 & 1.0 & 0.5 \\
					\bottomrule
				\end{tabular}
			\end{small}
		\end{center}
		\vskip -0.1in
	\end{table}

	R-MSNs \citep{lim2018forecasting}, can also be used to forecast treatment responses for an arbitrary number of steps in the future. In our paper we focus on one-step ahead predictions of the treatment responses. However, the Time Series Deconfounder can also be applied to estimate the effects of a sequence of future treatments. 
	
	\newpage
	
	\section{Additional Results}
	
	\subsection{Experiments on Synthetic Data}
	
	\begin{figure*}[t]
		\begin{center}
			\centerline{\includegraphics[width=1\columnwidth]{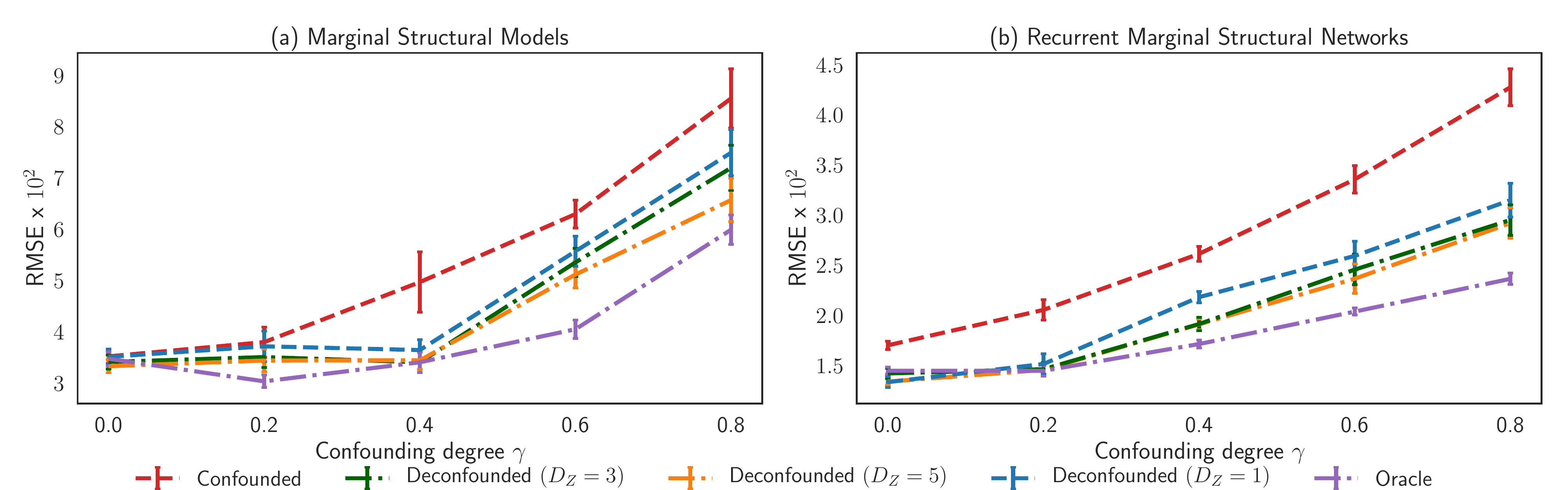}}
			\caption{Results for deconfounding one-step ahead estimation of treatment responses in two outcome models: (a) Marginal Structural Models (MSM) and (b) Recurrent Marginal Structural Networks (R-MSN). The simulated (true) size of the hidden confounders is $D_Z = 3$. The average RMSE and the standard error in the results are computed for 30 dataset simulations for each different degree of confounding, as measured by $\gamma$. }
			\label{fig:deconfounding_eval_3}
		\end{center}
		\vspace{-0.8cm}
	\end{figure*}
	
	We considered an additional experimental set-up where we have simulated hidden confounders of dimension $D_Z = 3$. In Figure \ref{fig:deconfounding_eval_3} we illustrate the root mean squared error (RMSE) for one-step-ahead estimation of treatment responses for patients in the test set without adjusting for the bias from the hidden confounders (Confounded), when using the simulated hidden confounders (Oracle) and after applying the Time Series Deconfounder with different model specifications (Deconfounded). We notice that the Time Series Deconfounder can still account for the bias from hidden confounders when the true size for the hidden confounders is underestimated in the factor model and set to ($D_Z = 1$). The performance is improved when setting $D_Z$ to the true number of hidden confounders or when overestimating the number of hidden confounders.

	\subsection{Model of Tumour Growth}
	
	To show the applicability of our method in a more realistic simulation, we use the pharmacokinetic-pharmacodynamic (PK-PD) model of tumor growth under the effects of chemotherapy and radiotherapy proposed by \citet{geng2017prediction}. The tumor volume after $t$ days since diagnosis is modeled as follows:
	\begin{equation}
	\begin{aligned}
	V(t) ={}  \Big( 1 +  \underbrace{\rho \text{log}(\dfrac{K}{V(t-1)})}_{\text{Tumor growth}} - \underbrace{\beta_c C(t)}_{\text{Chemotherapy}} - 
	\underbrace{\left(\alpha_r d(t) + \beta_r d(t)^2\right)}_{\text{Radiotherapy}} + \underbrace{e_t}_{\text{Noise}}  \Big)V(t-1) 
	\end{aligned}
	\end{equation}
	where  $K, \rho, \beta_c, \alpha_r, \beta_r, e_t $ are sampled as described in \citet{geng2017prediction}. $C(t)$ is the chemotherapy drug concentration and $d(t)$ is the dose of radiation. Chemotherapy and radiotherapy prescriptions are modeled as Bernoulli random variables that depend on the tumor size. Full details about treatments are in \citet{ lim2018forecasting}. 
	
	\begin{table}[ht]
		\caption{Average  RMSE $\times 10^2$  (normalised by the maximum tumour volume) and the standard error in the results for predicting the effect of chemotherapy and radiotherapy on the tumour volume.}
		\label{tab:pk_pd}
		\begin{center}
			\begin{small}
				\begin{tabular}{c|cc}
					\toprule
					Outcome model & MSM & R-MSN  \\
					\hline
					Confounded  & 7.29 $\pm$ 0.14   &  5.31 $\pm$ 0.16   \\
					Deconfounded ($D_Z = 1$) &6.47 $\pm$ 0.16  & 4.76 $\pm$ 0.17 \\
					Deconfounded ($D_Z = 5$)  & 6.25 $\pm$ 0.14 & 4.79 $\pm$ 0.19   \\
					Deconfounded ($D_Z = 10$)  & 6.31 $\pm$ 0.11  & 4.54 $\pm$ 0.17 \\
					\hline
					Oracle &  6.92 $\pm 0.19$  & 5.00 $\pm$ 0.15 \\
					\bottomrule
				\end{tabular}
			\end{small}
		\end{center}
	\end{table}
	
    To account for patient heterogeneity due to genetic features \citep{bartsch2007genetic}, the prior means for $\beta_c$ and  $\alpha_r$ are adjusted according to three patient subgroups as described in \citet{lim2018forecasting}. The patient subgroup $S^{(i)} \in \{1, 2, 3\}$ represents a confounder because it affects the tumor growth and subsequently the treatment assignments. We reproduced the experimental set-up in \citet{lim2018forecasting} and simulated datasets with 10000 patients for training, 1000 for validation, and 1000 for testing. We simulated 30 datasets and averaged the results for testing the MSM and R-MSN outcome models without the information about patient types (confounded),  with the true simulated patient types, as well as after applying the Time Series Deconfounder with $D_Z \in \{1, 5, 10\}$.  
    
    The results in Table \ref{tab:pk_pd} indicate that our method can infer substitutes for static hidden confounders such as patient subgroups which affect the treatment responses over time.  By construction, $\hist{Z}_t$ also captures time dependencies which help with the prediction of outcomes. This is why the performance of the deconfounded models is slightly better than of the oracle model which uses static patient groups. 
	
	\subsection{MIMIC III}
	
	We performed an additional experiment using the dataset extracted from the MIMIC III database where we have removed 3 patient covariates from the dataset (temperature, glucose, hemoglobin). In Table \ref{tab:mimic_rm} we report the results for estimating the effects of antibiotics, vasopressors, and mechanical ventilator on the patient's white blood cell count when including all variables, after removing these 3 patient covariates (which we notice that further confound the results) and after applying the Time Series Deconfounder with different settings for $D_Z$. 
	
	\begin{table*}[ht]
		\caption{Average  RMSE $\times 10^2$  and the standard error in the results for predicting the effect of antibiotics, vasopressors, and mechanical ventilator on white blood cell count. The results are for 10 runs.}
		\label{tab:mimic_rm}
		\begin{center}
			\begin{small}
				\begin{tabular}{c|cc}
					\toprule
					& \multicolumn{2}{c}{White blood cell}  \\
					Outcome model  & MSM & R-MSN   \\
					\hline
					All patient covariates  & $3.90 \pm 0.00$  &  $2.91 \pm 0.05$  \\
					Removed 3 covariates  & $4.12 \pm 0.00$  &  $3.11 \pm 0.03$  \\
					\hline
					Deconfounded ($D_Z = 1$) & $3.98 \pm 0.02$  & $3.05 \pm 0.05$ \\
					Deconfounded ($D_Z = 3$)  &  $3.91 \pm 0.03$ & $2.87 \pm 0.08$  \\
					Deconfounded ($D_Z = 5$)  & $3.85 \pm 0.04$  & $2.81 \pm 0.03$ \\
					\bottomrule
				\end{tabular}
			\end{small}
		\end{center}
		\vspace{-0.4cm}
	\end{table*}
	
	\section{Discussion}
	
	The Time Series Deconfounder firstly builds a factor model to infer substitutes for the multi-cause hidden confounders. If Assumption 3 holds and the fitted factor model captures well the distribution of the assigned causes, which can be assessed through predictive checks, the substitutes for the hidden confounders help us obtain sequential strong ignorability (Theorem 1).  Then, the Time Series Deconfounder uses the inferred substitutes for the hidden confounders in an outcome model that estimates individualized treatment responses. The experimental results show the applicability of the Time Series Deconfounder both in a controlled simulated setting and in a real dataset consisting of electronic health records from patients in the ICU. In these settings, the Time Series Deconfounder was able to remove the bias from hidden confounders when estimating treatment responses conditional on patient history. 
	
	In the static causal inference setting, several methods have been proposed to extend the deconfounder algorithm in \citet{wang2018blessings}. For instance, \citet{wang2019multiple} augment the theory in the deconfounder algorithm in \citet{wang2018blessings} by extending it to causal graphs and show that by using some of the causes as proxies of the shared confounder in the outcome model one can identify the effects of the other causes. \citet{d2019multi} also suggests using proxy variables to obtain non-parametric identification of the mean potential outcomes \cite{miao2018identifying}. Additionally, \citet{kong2019multi} proves that identification of causal effects is possible in the multi-cause setting when the treatments are normally distributed and the outcome is binary and follows a logistic structural equation model. 
	
	For the Time Series Deconfounder, similarly to \citet{wang2018blessings}, identifiability can be assessed by computing the uncertainty in the outcome model estimates, as described in Section 4.2. When the treatment effects are non-identifiable, the Time Series Deconfounder estimates will have high variance. Thus, future work could explore building upon the results in \citet{wang2019multiple} and \citet{d2019multi} and using proxy variables in the outcome model to prove identifiability of causal estimates in the multi-cause time-series setting.
	
	\newpage
	\bibliography{refs_appendix}
	\bibliographystyle{icml2020}
\end{document}